\documentclass{article}
\usepackage[preprint]{spconf}

\usepackage{amsmath,amsfonts,amscd,amsthm}
\usepackage{graphicx,subfigure}

\usepackage{hyperref}

\usepackage{algorithm,algpseudocode}

\usepackage[utf8]{inputenc}

\newcommand{\rk}{\operatorname{rk}}
\newcommand{\spn}{\operatorname{span}}

\newcommand{\calP}{\mathcal{P}}
\newcommand{\calS}{\mathcal{S}}

\newtheorem{Satz}{Theorem}[section]
\newtheorem{Thm}[Satz]{Theorem}

\newtheorem{Prop}[Satz]{Proposition}

\newtheorem{Prob}[Satz]{Problem}

\newtheorem{Def}[Satz]{Definition}

\newtheorem{Not}[Satz]{Notations}

\newtheorem{Ass}[Satz]{Assumption}

\newcommand {\vv} {{\mbox{ $v$}}}
\newcommand {\vu} {{\mbox{ $u$}}}

\title{APPROXIMATE RANK-DETECTING FACTORIZATION OF LOW-RANK TENSORS}
\twoauthors
  {Franz J. Kir\'{a}ly \thanks{Copyright 2013 IEEE (tentative, in case of acceptance). Submitted to the IEEE 2013 International Conference on Acoustics, Speech, and Signal Processing (ICASSP 2013), scheduled for 26-31 May 2013 in Vancouver, British Columbia, Canada. Personal use of this material is permitted. However, permission to reprint/republish this material for advertising or promotional purposes or for creating new collective works for resale or redistribution to servers or lists, or to reuse any copyrighted component of this work in other works, must be obtained from the IEEE. Contact: Manager, Copyrights and Permissions / IEEE Service Center / 445 Hoes Lane / P.O. Box 1331 / Piscataway, NJ 08855-1331, USA. Telephone: + Intl. 908-562-3966.}}
	{\small Machine Learning Group, Computer Science \\
     \small    Berlin Institute of Technology (TU Berlin)\\
	\small Franklinstr.~28/29, 10587 Berlin, Germany\\
	  \small and Discrete Geometry Group, Institute of Mathematics, FU Berlin \\
        \small and Mathematisches Forschungsinstitut Oberwolfach}
  {Andreas Ziehe}
	{\small  Machine Learning Group, Computer Science \\
     \small    Berlin Institute of Technology (TU Berlin)\\
\small 	Franklinstr.~28/29, 10587 Berlin, Germany\\
\phantom{and XX}\\
\phantom{more XY}}
\begin{document}
%
\maketitle
\begin{abstract}
  We present an algorithm, AROFAC2, which detects the (CP-)rank of a
  degree $3$ tensor and calculates its factorization into rank-one
  components. We provide generative conditions for the algorithm to
  work and demonstrate on both synthetic and real world data that
  AROFAC2 is a potentially outperforming alternative
  to the gold standard PARAFAC over
  which it has the advantages that it can intrinsically detect the true rank,
  avoids spurious components, and is stable
  with respect to outliers and non-Gaussian noise.
\end{abstract}
\section{Introduction}
\label{sec:intro}

Polyadic decomposition of tensors into their canonical components (= canonic polyadic resp.~CP-decomposition) and determining the number of those (= the rank) is a multidimensional generalization of the Singular Value Decomposition and the matrix rank, and a reoccurring task in all practical sciences, appearing many times under different names; first discovered by Hitchcock~\cite{Hitchcock1927} and then re-discovered under names such as PARAFAC~\cite{Harshman1970} or CANDECOMP~\cite{CarrollChang1970}, it has been applied in many fields such as chemometrics, psychometrics, and signal processing \cite{Bro1997,Parafac2000,NionSidi2009}. An extensive survey of many applications can be found in \cite{Sidi2004,DeLathauwer:2008}.

Considerable effort has been devoted to develop theory and methodology for the CP-decomposition, however many fundamental issues are still unresolved. The mathematical theory concerning CP-decompositions of tensors which are not matrices is only partly understood; also, while there exist several methods to calculate the CP-decomposition of a tensor~\cite{TomasiBro2006,ND08}, they are extrinsical in the sense that a structure-agnostic loss function is optimized and also highly sensitive to outliers or non-Gaussian noise - problems which have been heuristically attempted to cope with (e.g.~\cite{CEM1208}). Moreover, determining the rank of a noisy tensor remains a problematic task despite the existence of heuristics~\cite{BroKiers2003}.

In this paper, we present AROFAC2, a method for calculating the CP-decomposition of a low-rank degree $3$ tensor and its rank, which is based on theoretical considerations and intrinsical calculations making use of the algebraic structure of degree $3$ tensors, part of which have already surfaced in~\cite{Arofac1}. Specifically, we show how the algebraic structure can be used to obtain components one-by-one by alternating projections - a technique which draws inspiration from \cite{Cardoso92} - and how to reduce determination of rank to a clustering problem. Due to its structure-awareness, our algorithm only finds the numerically stable components while avoiding spurious ones, and determines the correct rank; it is also less sensitive to outliers or noise. We demonstrate its superiority to existing approaches on synthetic data and a chemometrics data set.

\section{Tensor Factorization and Simultaneous SVD}
\label{sec:theory}
We briefly review the basic definitions of rank-one tensor decomposition, and introduce some notation.
\begin{Not}
The set of $(n_1\times n_2\times n_3)$-tensors of degree $3$ is denoted by $\mathbb{C}^{n_1\times n_2\times n_3}.$
For $A\in\mathbb{C}^{n_1\times n_2\times n_3},$ the matrices
$$A_1,\dots,A_{n_3}\quad\mbox{with}\;A_k=(a_{ijk})_{\begin{subarray}{l}
        1\le i\le n_1\\  1\le j\le n_2\end{subarray}}$$
are called the $3$-slices of $A$.
\end{Not}

\begin{Def}\label{Def:tensorrank}
Let $A\in\mathbb{C}^{n_1\times n_2\times n_3}.$ Then, a decomposition
$$A=\sum_{i=1}^r u_i\otimes v_i\otimes w_i\quad\mbox{with}\;u_i\in\mathbb{C}^{n_1},v_i\in\mathbb{C}^{n_2},w_i\in\mathbb{C}^{n_3},$$
is called a rank $r$ canonic polyadic decomposition (or CP-decomposition) of $A$. The $u_i,v_i,w_i$ are called (rank-one-)components of $A$. The $u_i$ are called mode-$1$-, the $v_i$ are called mode-$2$-, and the $w_i$ mode-$3$-components. The tensors $u_i\otimes v_i\otimes w_i$ are called (rank-one-)factors.

The (CP-)rank of $A$, denoted by $\rk (A),$ is the smallest $r$ such that $A$ has a rank $r$ CP-decomposition.
\end{Def}

This paper proposes an algorithmic solution for the following problem:

\begin{Prob}\label{Prob:CPfac}
Let $A\in\mathbb{C}^{n_1\times n_2\times n_3}.$ Determine $r=\rk (A)$ and a rank $r$ CP-decomposition of $A$.
\end{Prob}

and for its approximate version, i.e., the case where $A$ is noisy but of low tensor rank, and one wants to find a CP-decomposition of the noiseless $A$.

Related to that is the problem of finding a simultaneous singular value decomposition (SVD):
\begin{Prob}\label{Prob:simsvd}
Let $A_1,\dots, A_{n_3}\in\mathbb{C}^{n_1\times n_2}.$ Determine a rank $r,$ matrices $U\in\mathbb{C}^{n_1\times r},V\in\mathbb{C}^{n_2\times r}$ and diagonal matrices $W_1,\dots, W_{n_3}\in\mathbb{C}^{r\times r}$ such that
$A_k=U\cdot W_k\cdot V^\top$ 
for all $1\le k\le n_3.$
\end{Prob}

In fact, Problems~\ref{Prob:CPfac} and~\ref{Prob:simsvd} are equivalent. We briefly prove that and add one additional characterization which will be important in the application:
\begin{Prop}\label{Prop:eqprobs}
Let $A\in\mathbb{C}^{n_1\times n_2\times n_3},$ let\\ $A_1,\dots, A_{n_3}\in\mathbb{C}^{n_1\times n_2}$ be the $3$-slices of $A$. Then, the following are equivalent:\\
(i) $A$ is a tensor of rank $r$ in the sense of Definition~\ref{Def:tensorrank}.\\
(ii) The $A_i$ have a simultaneous SVD of rank $r$ in the sense of Problem~\ref{Prob:simsvd}.\\
(iii) There are rank one matrices $M_1,\dots, M_r\in\mathbb{C}^{n_1\times n_2}$ and coefficients $\lambda_{ij}\in\mathbb{C}, 1\le i\le n_3, 1\le j\le r,$ such that
$$A_i=\sum_{j=1}^r \lambda_{ij} M_j\quad\mbox{for all}\;1\le i\le n_3.$$
\end{Prop}
\begin{proof}
(i)$\Leftrightarrow$(ii): There is only a notational difference which results from writing the columns of $U$ as $u_i$, the columns of $V$ as $v_i$, and, for $i$ fixed and $k$ running, the vectors formed by the $i$-th diagonal entries of the $W_k$ as $w_i$, and the $A_k$ as the $3$-slices of $A$.\\
(i)$\Leftrightarrow$(iii): Again, the difference is only notational. Write $M_i=u_i\otimes v_i$ and $\lambda_{ij}$ the $i$-th component of $w_j$.
\end{proof}

\section{Uniqueness of CP-decomposition and linear factoring}
\label{sec:uniq}
By the definitions above, it is not a-priori clear whether the CP-decomposition, including terminology such as component or factor, is well-defined, in particular whether it is unique. The method we present in this paper exploits certain properties of tensors with low rank. The main assumption on the tensor $A$ to factor is the following:
\begin{Ass}
$A\in\mathbb{C}^{n_1\times n_2\times n_3},\rk(A)\le \min(n_1,n_2,n_3).$
\end{Ass}
It is probably known or folklore that in this case the CP-decomposition is unique, nevertheless we were not able to retrieve an exact reference. Thus, we provide a proof instead (for the notion of genericity, consult the appendix of \cite{Kir12JMLR}):
\begin{Thm}\label{Thm:uniq}
Let $A\in\mathbb{C}^{n_1\times n_2\times n_3}$ be generic with rank $\rk(A) = r \le \min(n_1,n_2,n_3),$ let
$$A=\sum_{i=1}^r u_i\otimes v_i\otimes w_i\quad\mbox{with}\;u_i\in\mathbb{C}^{n_1},v_i\in\mathbb{C}^{n_2},w_i\in\mathbb{C}^{n_3}$$
be a CP-decomposition of $A$. Then any CP-decomposition of $A$ can be obtained by replacing the $u_i,v_i,w_i$ by\\ $\lambda_iu_i,\nu_iv_i,\lambda_i^{-1}\nu_i^{-1} w_i$, where $\lambda_i,\nu_i\in\mathbb{C}^\times$ for $1\le i\le r.$
\end{Thm}
\begin{proof}
Keep the notation from Proposition~\ref{Prop:eqprobs}, including $A_i$ and $M_j$. Due to the observation in the proof of Proposition~\ref{Prop:eqprobs} that (i) and (iii) differ only in notation, the statement of this theorem is equivalent to proving that there is only one unique way to present the $A_i$ as
$$A_i=\sum_{j=1}^r \lambda_{ij} M_j\quad\mbox{for all}\;1\le i\le n_3.$$
with rank one matrices $M_j$ and numbers $\lambda_{ij}$, up to renumbering and the obvious rescaling by replacing $M_j$ with $\mu_jM_j$ and $\lambda_{ij}$ with $\mu_j^{-1} M_j.$ Now since $A$ is generic, the $M_j$ are generic rank-one matrices, and the $\lambda_{ij}$ are generic numbers. Thus, interpreting the $A_i$ as rows of a $(n_3\times n_1n_2)$-matrix $\tilde{A}$, the $M_j$ as rows of a $(r\times n_1n_2)$-matrix $M$, and the $\lambda_{ij}$ as elements of a $(n_3\times r)$-matrix $\Lambda$. The presentation above can be reformulated as $\tilde{A}=\Lambda M.$ The assumption $n_3\le r$ and Proposition~\ref{Prop:eqprobs} (iii) thus imply that the $M_j$ lie in the span of the $A_i$. Now since the $\lambda_{ij}$ are generic, $\Lambda$ is a completely generic matrix. Thus, a different presentation of $\tilde{A}$ would correspond to the existence of an invertible $(r\times r)$ matrix $P$ such that $\tilde{A}=\Lambda P^{-1} M'$ with a matrix $M'$ whose rows correspond to rank-one-matrices, i.e., $M'=PM.$ But a linear combination
$$M'_i=\sum_{i=1}^{r} p_{ij} M_j$$
of the $(n_1\times n_2)$ matrices $M_j$ has rank one if and only if exactly one of the $p_{ij}$ (with $i$ fixed) is non-zero, since $r\le \min (n_1,n_2)$. Since $P$ is of full rank, this implies that $P$ is the product of a $(r\times r)$ permutation matrix with a full rank diagonal $(r\times r)$ matrix. But this is, as stated above, equivalent to the statement to prove.
\end{proof}

Thus, under our assumptions components are unique up to scaling and numbering, and factors are unique up to numbering.

Proposition~\ref{Prop:eqprobs}, together with the uniqueness guarantee in Theorem~\ref{Thm:uniq}, gives the following statement:
\begin{Prop}\label{Prop:arofac}
Let $A\in\mathbb{C}^{n_1\times n_2\times n_3}$ be generic with rank $\rk(A) = r \le n_3,$ let $A_1,\dots, A_{n_3}\in\mathbb{C}^{n_1\times n_2}$ be the $3$-slices of $A$. Then, up to scaling, there are $r$ unique vectors $\lambda^{(1)},\dots,\lambda^{(r)}\in\mathbb{C}^{n_3}$ such that
$$M(\lambda^{(k)})=\sum_{i=1}^{n_3}\lambda_i^{(k)} A_i$$
has rank one. Moreover, for each $k$, let $M(\lambda^{(k)})=u_k\otimes v_k$ be the (rank-one-)SVD. Then, the $u_k$ and $v_k$ are mode-$1$- and mode-$2$-components of $A$, belonging to the same factor.
\end{Prop}

\section{The AROFAC2 Algorithm}
\label{sec:algo}

We propose an algorithm which computes rank and CP-decomposition of a degree $3$ tensor $A$. It uses the routine \texttt{FindRankOne} which finds mode-$1$- and -$2$-components which we will present first as Algorithm~\ref{Alg:findrkone}. In step~\ref{Alg:findrkone.step1}, the tensor is first decomposed into $3$-slices $A_i$. In step~\ref{Alg:findrkone.step2}, an approximate representation $V$ for their span is calculated. This can be a PCA of the $A_i$, i.e., principal values or components, or a numerical span of lower dimension (if the true rank of $A$ is known, the dimension should equal the rank). In step~\ref{Alg:findrkone.step3} a matrix $M\in V$ is randomized. This can be a random matrix in an exact span, or a matrix which is, e.g., sampled from a Gaussian with covariance matrix follows the estimated sample distribution of the $A_i$. Then steps~\ref{Alg:findrkone.step4} and~\ref{Alg:findrkone.step5},  are repeated until convergence of $M$ is attained. Step~\ref{Alg:findrkone.step4} takes (a possibly non-square) $M$ to its third power, magnifying its largest singular value and diminishing the others. Step~\ref{Alg:findrkone.step5} projects $M$ onto $V$ and normalizes the result. Projection can be achieved by exact projection, or re-scaling, e.g., according to the principal values of the $A_i$. After convergence is reached, $M$ will be approximately of rank one, and its first singular vectors, which are obtained in step~\ref{Alg:findrkone.step6}, will be estimates for a mode-$1$- and a mode-$2$-component of $A$.

\begin{algorithm}[ht]
\caption{\label{Alg:findrkone}
\texttt{FindRankOne}$(A)$
{\it Input: $A\in\mathbb{C}^{n_1\times n_2\times n_3}$}
{\it Output: One random mode-$1$-component $u\in\mathbb{C}^{n_1}$ and one random mode-$2$-component $v\in\mathbb{C}^{n_2}$ of $A$ in the same factor.}
}
\begin{algorithmic}[1]

    \State \label{Alg:findrkone.step1} Let $A_1,\dots, A_{n_3}$ be the $3$-slices of $A$.
    \State \label{Alg:findrkone.step2} Calculate an approximate representation $V$ of $\spn \langle A_1,\dots, A_{n_3}\rangle.$
    \State \label{Alg:findrkone.step3} Randomize $M\in V.$
    \Repeat
    \State \label{Alg:findrkone.step4} $M\leftarrow M\cdot M^\top\cdot M$
    \State \label{Alg:findrkone.step5} $M\leftarrow \calP_V (M)$
    \Until $M$ has converged
    \State \label{Alg:findrkone.step6} Calculate approximate rank-one SVD of $M=u\cdot v^\top$
    \State \label{Alg:findrkone.step7} Return $u,v$
\end{algorithmic}
\end{algorithm}

Algorithm~\ref{Alg:AROFAC2} uses Algorithm~\ref{Alg:findrkone} to obtain a full CP-decomposition and an estimate for the rank of $A$. In step~\ref{Alg:AROFAC2.step1}, several candidate estimates for components of all modes are obtained. Mode-$3$-components can be obtained by switching coordinates in $A$ (e.g., switch the first with the third). If the mode-$3$-components are not relevant for the problem at hand, e.g., in the setting of simultaneous SVD as in Problem~\ref{Prob:simsvd}, this can be omitted. In steps~\ref{Alg:AROFAC2.step2} and step~\ref{Alg:AROFAC2.step3}, a clustering algorithm is applied to estimate the number of cluster centers and cluster the candidate components. This can be done by different methods, or one single algorithm. Our implementation uses the mean shift algorithm which also estimates the number of cluster centers \cite{MeanShift02}. Since Algorithm~\ref{Alg:findrkone} links pairs of components, this information can then be applied in step~\ref{Alg:AROFAC2.step4} to the estimated cluster centers in order to link pairs to full triples (which is unnecessary if only the first two modes are considered), e.g., by majority vote or vote weighted by closeness of a pair to the cluster center. The corresponding decomposition is then presented as the estimated solution in step~\ref{Alg:AROFAC2.step5}.

\begin{algorithm}[ht]
\caption{\label{Alg:AROFAC2}
\texttt{AROFAC2}$(A)$\quad
{\it Input: $A\in\mathbb{C}^{n_1\times n_2\times n_3}$}\newline
{\it Output: Rank and approximate CP-decomposition of $A$.}
}
\begin{algorithmic}[1]
    \State \label{Alg:AROFAC2.step1} Repeat \texttt{FindRankOne}$(A)$ to find sets $\calS_1,\calS_2,\calS_3$ of potential mode-$1$-,-$2$-, and -$3$-components of $A$.
    \State \label{Alg:AROFAC2.step2} Use clustering algorithm to determine number $r$ of clusters for $\calS_1,\calS_2,$ and $\calS_3$
    \State \label{Alg:AROFAC2.step3} Cluster $\calS_1,\calS_2,$ and $\calS_3$ to obtain cluster centers $u_1,\dots, u_r$ of $\calS_1$, centers $v_1,\dots, v_r$ of $\calS_2$ and $w_1,\dots, w_r$ of $\calS_3.$
    \State \label{Alg:AROFAC2.step4} Use the information from \texttt{FindRankOne}$(A)$ to renumber the $u_i,v_j,w_k$ such that for each $\ell$, the components $u_\ell,v_\ell, w_\ell$ belong to the same factor.
    \State \label{Alg:AROFAC2.step5} Return $r$ as the rank of $A$ and $A=\sum_{i=1}^r u_i\otimes v_i\otimes w_i$ as its CP-decomposition.
\end{algorithmic}
\end{algorithm}

\section{Experiments}
First, we demonstrate our algorithm on simulated toy data.
The input tensor consists of $3$-slices
$A_k=\sum_{i=1}^r \lambda_{ik}\vu_i\vv_i^\top$
compare Proposition~\ref{Prop:eqprobs}.
Each slice is generated as follows: Exact singular vectors $\vu_i\in\mathbb{R}^{n_1}, \vv_i\in\mathbb{R}^{n_2}, 1\le i\le r$
are sampled independently and uniformly from the $n_1$-sphere and $n_2$-sphere, respectively.
The $\lambda_{ik}$ are sampled independently and uniformly from the standard normal distribution. Then, to each matrix $A_k,  1\le k\le n_3 $, noise is added in the form of a $(n_1\times n_2)$ matrix whose entries are independently sampled from a normal distribution with mean $0$ and covariance $\varepsilon\in\mathbb{R}^+.$
Figure \ref{fig:res_pm} shows the accuracy of the estimated $U$ and $V$ for $n_1=50 , n_2=60 , n_3=70 , r=10 , \varepsilon=0.1$ for PARAFAC  (top row) vs AROFAC2 (bottom row).
The estimation quality is essentially the same, however for AROFAC2  the correct
rank $r=10$ has been detected automatically.

\begin{figure}[htb]
     \begin{center}
        \subfigure[estimated coefficients]{%
  \label{fig:res_pm}
\includegraphics[width=3.5cm]{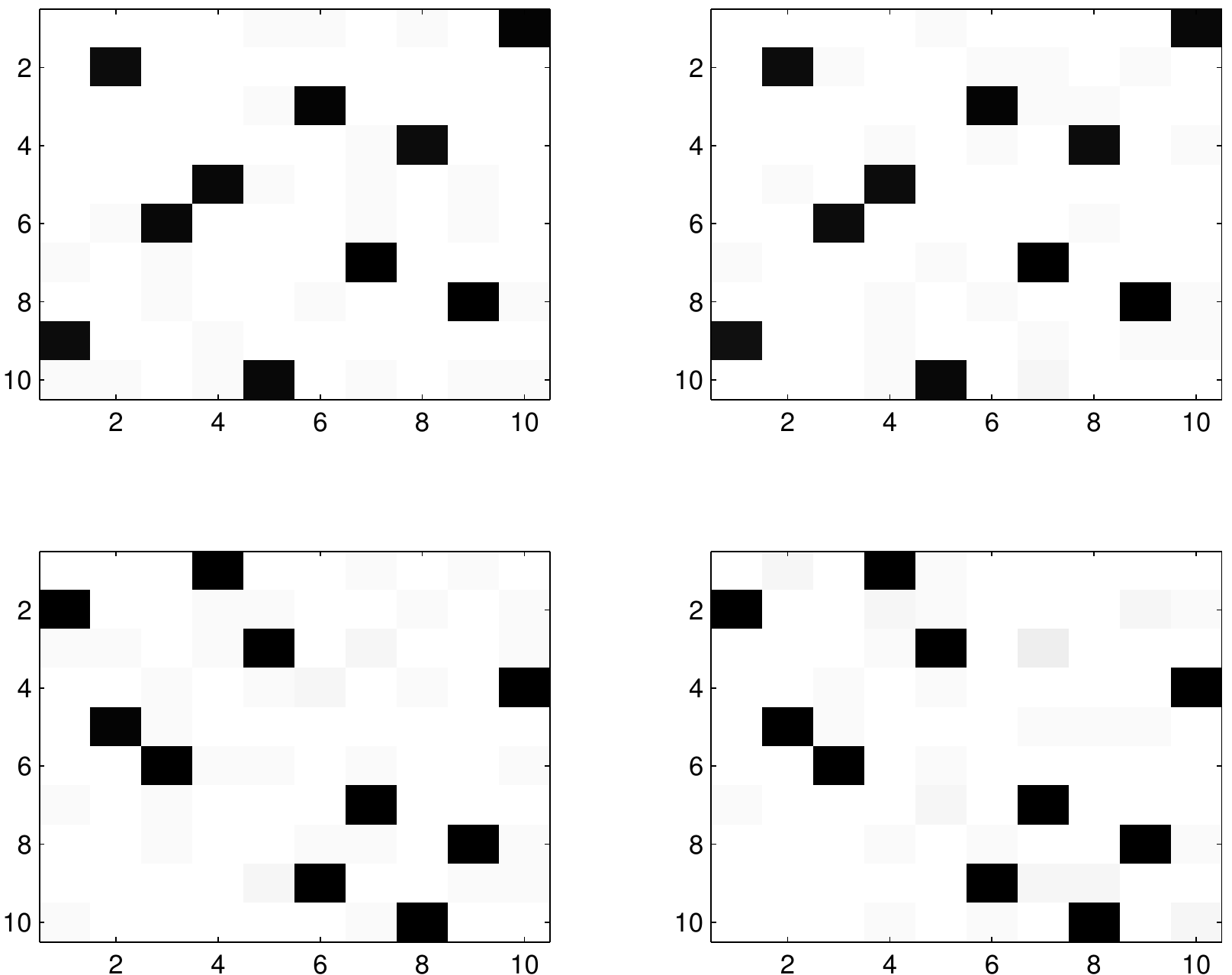}
        }~\subfigure[estimated rank]{%
  \label{fig:rank_est}
\includegraphics[width=4cm]{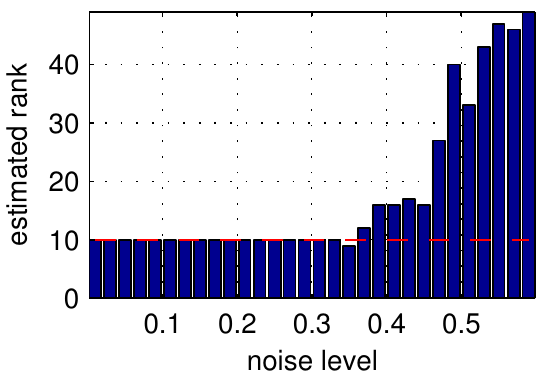}
        }

    \end{center}
  \caption{
(a) Absolute values of the correlation coefficients between the true and estimated components for PARAFAC (top row) vs AROFAC2 (bottom row). The correct solution has the same support as a permutation matrix.\;
(b) Estimated rank for increasing noise levels $\varepsilon$. True rank is $r=10$.
  }
   \label{fig:synthexps}
\end{figure}

%
In a second numerical evaluation, we analyze the noise robustness of the AROFAC2 algorithm. The data is generated as before, but the noise level  $\varepsilon$ is increased from $\varepsilon=0.01$ to  $\varepsilon=0.6$. In Figure \ref{fig:rank_est} we see, that the correct rank $r=10$ has been found for  $\varepsilon\le0.35$ and is mainly overestimated for larger noise levels.


Finally we apply our algorithm to a publicly available data set from
chemometrics.  The {\bf Dorrit fluorescence data} \cite{dorrit,RiuBro}
contains 27 synthetic samples of different mixtures of four analytes
(hydroquinone, tryptophan, phenylalanine and DOPA) that were measured
in a Perkin-Elmer LS50 B fluorescence spectrometer. The measurements
of emission spectra at multiple excitation wavelengths give rise to an
excitation-emission matrix (EEM) for each sample and thus form a degree 3
tensor which is known to obey the trilinear model where the rank is
determined by the number of fluorophores \cite{dorrit,RiuBro,CEM1208}.
As described in \cite{Engelen2011} this data set is highly suited to
assess the performance of different methods due to its realistic noise from the physical environment
and the availability of a priori knowledge of the underlying components.
Figure \ref{fig:res} shows the estimated emission and excitation
spectra for PARAFAC with 4 and 5 components and AROFAC2 with
auto-detected 5 components.
We note that the results of AROFAC2 are in excellent agreement with the known spectra (cf. Figs.~2 and 3 in  \cite{CEM1208}) while the components found by PARAFAC lack accuracy. In addition we observe that the AROFAC2 loadings better fulfill the non-negativity constraints even though they were not enforced explicitly.
Furthermore, the peak of the fifth component around 315 nm in both excitation and emission spectra which can be attributed to Rayleigh scatter in all samples \cite{Engelen2007,CEM1208} is sharper and thus in better agreement with its expected shape when identified by AROFAC2.

\begin{figure}[htb]
\begin{minipage}[b]{1.0\linewidth}
 \centerline{\includegraphics[width=8.0cm]{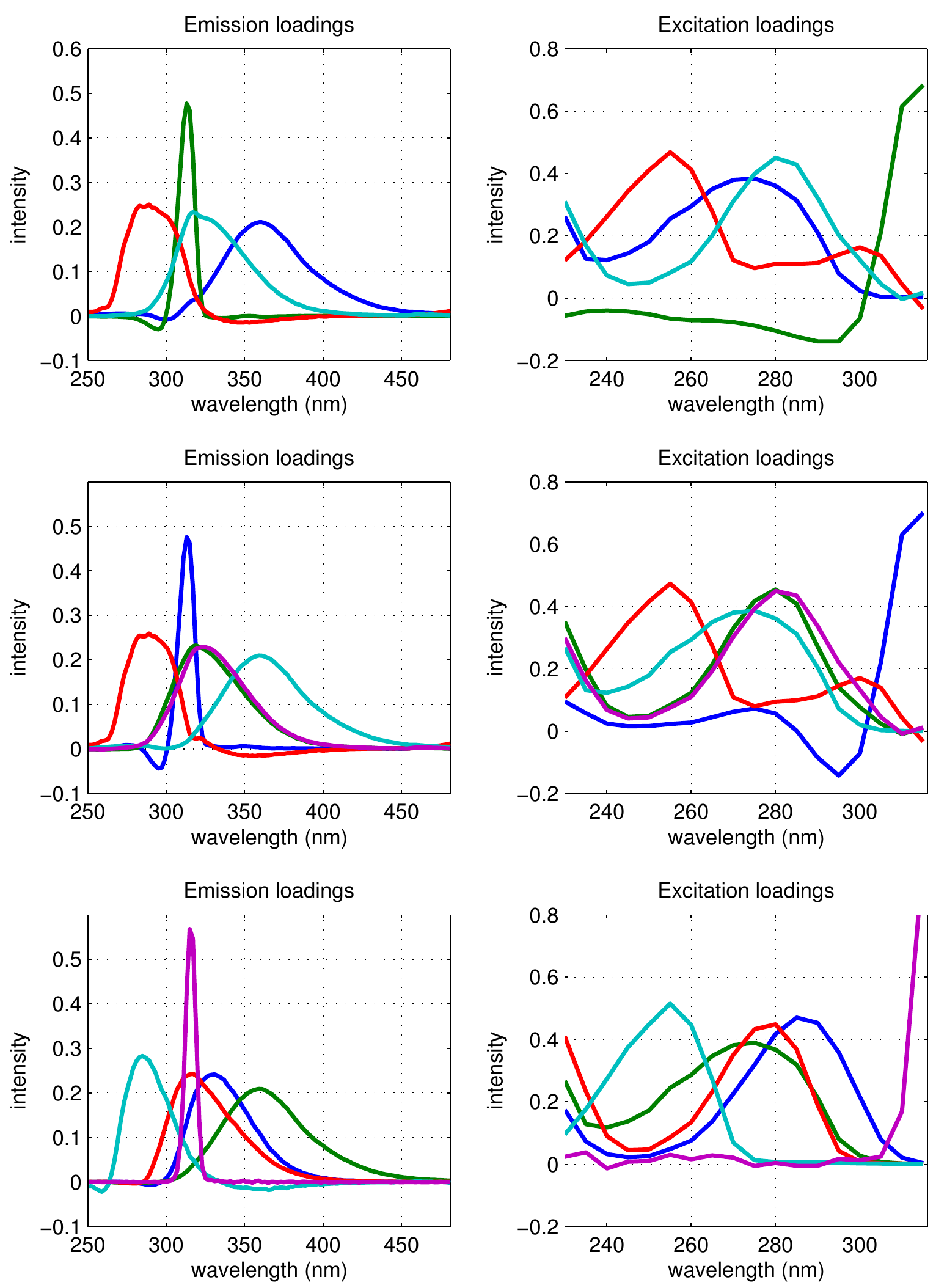}}

\end{minipage}
\caption{Emission (left) and excitation (right) spectra for the Dorrit data using PARAFAC with $r=4$ (top row)  PARAFAC with $r=5$ (middle row) and AROFAC2 with  $r=5$ automatically detected (bottom row). }
\label{fig:res}
\end{figure}


\section{Conclusion}
\label{sec:concl}

With AROFAC2, we have presented an algorithm which can determine the CP-decomposition and the rank of a potentially noisy degree three tensor. We argue that due to how the algorithm is constructed, spurious and unstable components in the decomposition are not found; since the convergence criterion intrinsically enforces stability and thus informativeness of any found component. Our simulations demonstrate that AROFAC2 is competitive to the state-of-the-art method PARAFAC, but without the need to provide the true rank in advance. Also, AROFAC2 outperforms PARAFAC on the Dorrit data set which shows that AROFAC2 is a method which is more stable to outliers and the influence of non-Gaussian noise.

AROFAC2 uses the intrinsic algebraic structure of a low-rank degree tensor in the calculations, as opposed to most standard methods such as PARAFAC which assume a model and try to fit it, agnostic of its inner structure. We thus argue that algorithms exploiting this structure are to prefer whenever available, and the proper starting point for any method approaching any problem with algebraic features. We emphasize the potential benefit from applying structural insights to construct structure-aware methods.


\newpage


\nocite{*}
\bibliography{icassp13}		

\begin{thebibliography}{10}

\bibitem{Hitchcock1927}
Frank~L. Hitchcock,
\newblock ``The expression of a tensor or a polyadic as a sum of products,''
\newblock {\em Journal of Mathematics and Physics}, vol. 6, pp. 164--189, 1927.

\bibitem{Harshman1970}
Richard~A. Harshman,
\newblock ``Foundations of the parafac procedure: Models and conditions for an
  "explanatory" multi-modal factor analysis,''
\newblock {\em UCLA Working Papers in Phonetics}, vol. 16, no. 84, pp. 1--84,
  1970.

\bibitem{CarrollChang1970}
J.~Douglas Carroll and Jih-Jie Chang,
\newblock ``Analysis of individual differences in multidimensional scaling via
  an n-way generalization of 'eckart–young' decomposition,''
\newblock {\em Psychometrika}, vol. 35, pp. 283--319, 1970.

\bibitem{Bro1997}
Rasmus Bro,
\newblock ``{PARAFAC}. tutorial and applications,''
\newblock {\em Chemometrics and Intelligent Laboratory Systems}, vol. 38, no.
  2, pp. 149 -- 171, 1997.

\bibitem{Parafac2000}
Nikos~D. {Sidiropoulos}, Rasmus {Bro}, and Georgios~B. {Giannakis},
\newblock ``{Parallel factor analysis in sensor array processing},''
\newblock {\em IEEE Transactions on Signal Processing}, vol. 48, pp.
  2377--2388, Aug. 2000.

\bibitem{NionSidi2009}
Dimitri Nion and Nikos~D. Sidiropoulos,
\newblock ``A {PARAFAC}-based technique for detection and localization of
  multiple targets in a mimo radar system,''
\newblock in {\em Proc. ICASSP '09}, 2009, pp. 2077--2080.

\bibitem{Sidi2004}
Nikos~D. Sidiropoulos,
\newblock ``Low-rank decomposition of multi-way arrays: a signal processing
  perspective,''
\newblock in {\em Sensor Array and Multichannel Signal Processing Workshop
  Proceedings}, 2004, pp. 52 -- 58.

\bibitem{DeLathauwer:2008}
Lieven De~Lathauwer, Pierre Comon, and Nicola Mastronardi,
\newblock ``Special issue on tensor decompositions and applications,''
\newblock {\em SIAM J. Matrix Anal. Appl.}, vol. 30, no. 3, pp. .7--.7, Sept.
  2008.

\bibitem{TomasiBro2006}
Giorgio Tomasi and Rasmus Bro,
\newblock ``A comparison of algorithms for fitting the {PARAFAC} model,''
\newblock {\em Computational Statistics \& Data Analysis}, vol. 50, no. 7, pp.
  1700--1734, 2006.

\bibitem{ND08}
Dimitri Nion and Lieven De~Lathauwer,
\newblock ``An enhanced line search scheme for complex-valued tensor
  decompositions. application in {DS-CDMA},''
\newblock {\em Signal Processing}, vol. 88, no. 3, pp. 749--755, 2008.

\bibitem{CEM1208}
Sanne Engelen, Stina Frosch, and Bert~M. Jorgensen,
\newblock ``A fully robust {PARAFAC} method for analyzing fluorescence data,''
\newblock {\em Journal of Chemometrics}, vol. 23, no. 3, pp. 124--131, 2009.

\bibitem{BroKiers2003}
Rasmus Bro and Henk~A.L. Kiers,
\newblock ``A new efficient method for determining the number of components in
  parafac models,''
\newblock {\em Journal of Chemometrics}, vol. 17, pp. 274--286, 2003.

\bibitem{Arofac1}
Franz~J. Kir{\'a}ly, Andreas Ziehe, and Klaus-Robert M\"uller,
\newblock ``An algebraic method for approximate rank one factorization of rank
  deficient matrices,''
\newblock in {\em Latent Variable Analysis and Signal Separation}, vol. 7191 of
  {\em LNCS}, pp. 272--279. Springer Berlin Heidelberg, 2012.

\bibitem{Cardoso92}
Jean-Fran{\c{c}}ois Cardoso,
\newblock ``Iterative techniques for blind source separation using only
  fourth-order cumulants,''
\newblock in {\em Proc. EUSIPCO}, 1992, pp. 739--742.

\bibitem{Kir12JMLR}
Franz~J. Kir{\'a}ly, Paul von B{\"u}nau, Frank Meinecke, Duncan Blythe, and
  Klaus-Robert M{\"u}ller,
\newblock ``Algebraic geometric comparison of probability distributions,''
\newblock {\em Journal of Machine Learning Research}, vol. 13, no. Mar, pp.
  855--903, 2012.

\bibitem{MeanShift02}
Dorin Comaniciu and Peter Meer,
\newblock ``Mean shift: A robust approach toward feature space analysis,''
\newblock {\em IEEE Trans. Pattern Anal. Mach. Intell.}, vol. 24, no. 5, pp.
  603--619, 2002.

\bibitem{dorrit}
Dorrit Baunsgaard,
\newblock {\em Factors affecting 3-way modelling (PARAFAC) of fluorescence
  landscapes},
\newblock Ph.D. thesis, Royal Veterinary and Agricultural University, Dept. of
  Dairy and Food Technology, Frederiksberg, Denmark, 1999.

\bibitem{RiuBro}
Jordi Riu and Rasmus Bro,
\newblock ``Jack-knife technique for outlier detection and estimation of
  standard errors in parafac models,''
\newblock {\em Chemometrics and Intelligent Laboratory Systems}, vol. 65, no.
  1, pp. 35 -- 49, 2003.

\bibitem{Engelen2011}
Sanne Engelen and Mia Hubert,
\newblock ``Detecting outlying samples in a parallel factor analysis model,''
\newblock {\em Analytica Chimica Acta}, vol. 705, no. 1–2, pp. 155 -- 165,
  2011.

\bibitem{Engelen2007}
Sanne Engelen, Stina~Frosch Møller, and Mia Hubert,
\newblock ``Automatically identifying scatter in fluorescence data using robust
  techniques,''
\newblock {\em Chemometrics and Intelligent Laboratory Systems}, vol. 86, no.
  1, pp. 35 -- 51, 2007.

\bibitem{DeLathauwer:2006}
Lieven De~Lathauwer,
\newblock ``A link between the canonical decomposition in multilinear algebra
  and simultaneous matrix diagonalization,''
\newblock {\em SIAM J. Matrix Anal. Appl.}, vol. 28, no. 3, pp. 642--666, Aug.
  2006.

\bibitem{ComonGLM08}
Pierre Comon, Gene~H. Golub, Lek-Heng Lim, and Bernard Mourrain,
\newblock ``Symmetric tensors and symmetric tensor rank,''
\newblock {\em SIAM J. Matrix Analysis Applications}, vol. 30, no. 3, pp.
  1254--1279, 2008.

\bibitem{veen01}
{Alle-Jan} van~der Veen,
\newblock ``Joint diagonalization via subspace fitting techniques,''
\newblock in {\em Proc. ICASSP}, 2001, vol.~5.

\end{thebibliography}

\bibliographystyle{IEEEbib}


\end{document}